\documentclass[runningheads]{llncs}
\usepackage[T1]{fontenc}
\usepackage{graphicx}
\usepackage{booktabs}
\usepackage[misc]{ifsym}

\usepackage{mwe}
\usepackage[utf8]{inputenc} % allow utf-8 input
\usepackage[T1]{fontenc}    % use 8-bit T1 fonts
\usepackage{hyperref}       % hyperlinks
\usepackage{url}            % simple URL typesetting
\usepackage{booktabs}       % professional-quality tables
\usepackage{amsfonts}       % blackboard math symbols
\usepackage{nicefrac}       % compact symbols for 1/2, etc.
\usepackage{microtype}      % microtypography
\usepackage{xcolor}         % colors

\usepackage{graphicx}
\usepackage{subfigure}
\usepackage{booktabs} %
\usepackage{tikz}
\usepackage{hyperref}
\usepackage{epstopdf}
\usepackage{amssymb}
\usepackage{mathtools}

\usepackage{bbm}
\usepackage{amsmath,amsfonts}

\def\eqref#1{equation~\ref{#1}}
\def\1{\bm{1}}

\DeclareMathAlphabet{\mathsfit}{\encodingdefault}{\sfdefault}{m}{sl}
\SetMathAlphabet{\mathsfit}{bold}{\encodingdefault}{\sfdefault}{bx}{n}

\usepackage[capitalize,noabbrev]{cleveref}
\usepackage{multirow, multicol}

\usepackage{algorithm}
\usepackage{algpseudocode}

\usepackage{pifont}% \newcommand{\cmark}{\ding{51}}%

\begin{document}

\title{Probabilistic Gaussian Homotopy: A Probability-Space Continuation Framework for Nonconvex Optimization}

\titlerunning{Probabilistic Gaussian Homotopy}
% If the full title of your paper is short enough to also fit in the running head, you can omit the abbreviated paper title here. You can check as follows: if you comment out the \titlerunning line, something will appear in the header of all odd-numbered pages of your PDF from page 3 onward. This something is either the full title (in which case all is well), or the error message "Title Suppressed Due to Excessive Length". If this error message appears, you're going to want to provide an abbreviated title within the \titlerunning command, because if you won't do it, Springer will do it for you.

%N.B.: Author information (both in the \author{} and \authorrunning{} command) should only be present in the Camera-Ready Version of your paper. The version that you initially submit for review, ought to be double-blind. So, when initially submitting your paper, use:

\toctitle{Probabilistic Gaussian Homotopy: A Probability-Space Continuation Framework for Nonconvex Optimization}

% \author{Author information scrubbed for double-blind reviewing}
% \author{Eshed Gal\inst{1} \and Samy Wu Fung\inst{2} \and Eldad Haber\inst{1}}
% \authorrunning{E. Gal et al.}

% \institute{
% University of British Columbia, Vancouver, BC, Canada
% \and
% Colorado School of Mines, Golden, CO, USA
% }
\author{Eshed Gal\inst{1}(\Letter) \and Samy Wu Fung\inst{2} \and Eldad Haber\inst{1}}
\authorrunning{E. Gal et al.}

% 3. Include email addresses in the institute block
\institute{
University of British Columbia, Vancouver, BC, Canada \\
\email{eshedg@cs.ubc.ca}, \email{ehaber@eoas.ubc.ca}
\and
Colorado School of Mines, Golden, CO, USA \\
\email{swufung@mines.edu}
}

\tocauthor{Eshed Gal, Samy Wu Fung, Eldad Haber}

% \author{Andr\'e Lauren Benjamin\inst{1} \and
% Calvin Cordozar Broadus Jr.\inst{2,3} \corr \and
% Antwan Andr\'e Patton\inst{1}\orcidID{0000-1111-2222-3333}}
% \author{Eshed Gal \and
% Samy Wu Fung \and
% Eldad Haber}
% You may leave out the orcidID information, if you want to.
% Use \corr to indicate the corresponding author. Note the spacing around the \corr command. Only one author can be the corresponding author.

%N.B.: comment out the \authorrunning{} command for the double-blind version of your paper submitted for review. Later, if your paper is accepted, use the command for the Camera-Ready Version.
% \authorrunning{A.L. Benjamin et al.}
% If there are two authors, write 'A.L. Benjamin and C.C. Broadus Jr.'
% If there are more than two authors, '[...] et al.' is used.

% \institute{Fictional Southern University, Savannah GA 31404, USA \email{\{a.l.benjamin,a.a.patton\}@fsu.fake}
% \and
% Fictional West Coast University, Long Beach CA 90840, USA \email{ccb@fwcu.fake}
% \and
% Secondary European Affiliation, Tiergartenstr. 17, 69121 Heidelberg, Germany
% \email{lncs@springer.com}}

\maketitle              % typeset the header of the contribution

\begin{abstract}
% We introduce Probabilistic Gaussian Homotopy (PGH), a probability-space continuation framework for nonconvex optimization. Unlike classical Gaussian homotopy, which smooths the objective and uniformly averages gradients, PGH deforms the associated Boltzmann distribution and induces Boltzmann-weighted aggregation of perturbed gradients, which exponentially biases descent directions toward low-energy regions. We show that PGH corresponds to a finite-temperature, posterior-mean generalization of the Moreau envelope and derive a dynamical system governing minimizer evolution along an annealed homotopy path. This establishes a principled connection between Gaussian continuation, Bayesian denoising, and diffusion-style smoothing. We further propose Probabilistic Gaussian Homotopy Optimization (PGHO), a practical stochastic algorithm based on Monte Carlo gradient estimation, and demonstrate strong performance on high-dimensional nonconvex benchmarks and sparse recovery problems where classical gradient methods and objective-space smoothing frequently fail.

We introduce Probabilistic Gaussian Homotopy (PGH), a probability-space continuation framework for nonconvex optimization. Unlike classical Gaussian homotopy, which smooths the objective and uniformly averages gradients, PGH deforms the associated Boltzmann distribution and induces Boltzmann-weighted aggregation of perturbed gradients, which exponentially biases descent directions toward low-energy regions. We show that PGH corresponds to a log-sum-exp (soft-min) homotopy that smooths a nonconvex objective at scale $\lambda>0$ and recovers the original objective as $\lambda\to 0$, yielding a posterior-mean generalization of the Moreau envelope, and we derive a dynamical system governing minimizer evolution along an annealed homotopy path. This establishes a principled connection between Gaussian continuation, Bayesian denoising, and diffusion-style smoothing. We further propose Probabilistic Gaussian Homotopy Optimization (PGHO), a practical stochastic algorithm based on Monte Carlo gradient estimation, and demonstrate strong performance on high-dimensional nonconvex benchmarks and sparse recovery problems where classical gradient methods and objective-space smoothing frequently fail.

\end{abstract}

\section{Introduction}

Many problems in science and engineering can be formulated as the minimization of a highly nonconvex but differentiable objective
\begin{equation}
\min_{x \in \mathbb{R}^n} f(x),
\end{equation}
where the landscape of $f$ may contain numerous local minima, flat regions, and narrow basins of attraction.
While local optimization methods are effective when initialized near a desirable solution, they often fail in strongly nonconvex settings.
Global strategies such as simulated annealing~\cite{bertsimas1993simulated}, evolutionary methods~\cite{michalewicz1996heuristic}, or multi-start schemes~\cite{marti2013multi} attempt to address this challenge, but frequently suffer from slow convergence or poor scalability in high dimensions.

A classical approach for handling nonconvexity is \emph{homotopy continuation}~\cite{dunlavy2005homotopy}.
Rather than solving the difficult problem directly, one constructs a family of progressively less smoothed objectives and tracks minimizers as the smoothing is removed.
Gaussian homotopy methods smooth the objective via convolution with a Gaussian kernel and gradually decrease the smoothing parameter~\cite{mobahi2015theoretical,mobahi2015link}.
Although effective in certain regimes, these methods operate in \emph{objective space} and lack a direct connection to modern probabilistic modeling.

In parallel, score-based diffusion models have emerged as a powerful framework for high-dimensional generative modeling.
These models define a family of progressively less noisy probability distributions and transport samples from a simple base density to a complex target distribution through a reverse-time stochastic differential equation.
While primarily viewed as sampling algorithms, diffusion models implicitly define a \emph{Gaussian homotopy in probability space}.

% \begin{figure}[t]
%     \centering
%     \includegraphics[width=0.95\linewidth]{Figures/ghomotopy.png}
%     \caption{{Gaussian homotopy for the function $f(x_1,x_2) = \log(Rosen(x_1,x_2) + \alpha (1 + \sin(3 x_1)\sin(3x_2))) $ where $Rosen$ is the Banana Rosenberg function. We compute the function by Monte-Carlo sampling. The function starts as a flat surface and gradually becomes the difficult multi-modal function that is difficult to optimize.}}
%     \label{fig:figpGH}
% \end{figure}
\begin{figure}[t]
    \centering
    \includegraphics[width=0.95\linewidth]{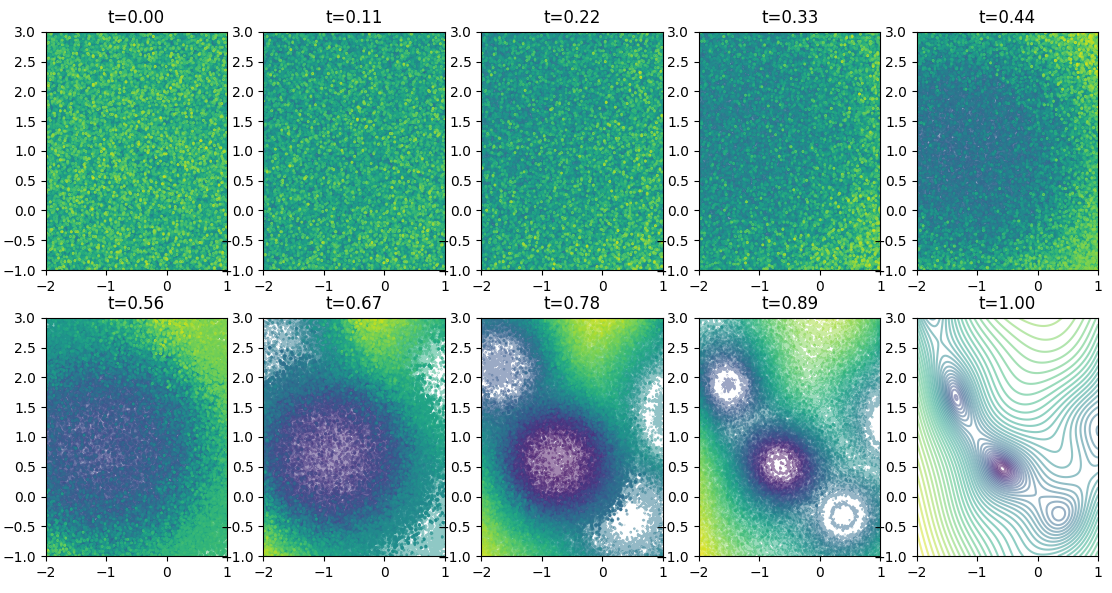}
    \caption{Gaussian homotopy for the function $f(x_1,x_2) = \log(Rosen(x_1,x_2) + \alpha (1 + \sin(3 x_1)\sin(3x_2)))$, where $Rosen$ denotes the banana Rosenbrock function. We estimate the homotopy-smoothed objective by Monte Carlo sampling. The surface begins nearly flat and gradually deforms into a highly multimodal landscape that is much harder to optimize.}
    \label{fig:figpGH}
\end{figure}

This observation motivates the central idea of this work.
Minimizing an objective $f(x)$ is equivalent to maximizing its associated Boltzmann density
\begin{equation}
p(x) \propto \exp\!\left(-\frac{1}{\lambda} f(x)\right).
\end{equation}
Instead of smoothing $f$ directly, we smooth its induced probability distribution and track the evolution of high-probability regions as the smoothing vanishes.
This yields a continuation method performed in \emph{probability space}, fundamentally distinct from classical objective-space homotopy.

We formalize this idea through \emph{Probabilistic Gaussian Homotopy (PGH)}, which constructs a time-dependent family of Gaussian-smoothed Boltzmann densities via a soft-min (log-sum-exp) aggregation of perturbed objective values.
As the smoothing decreases, the associated energy deforms continuously from a nearly flat landscape to the original nonconvex objective (see Figure~\ref{fig:figpGH}).
Minimizers are tracked along this path using a gradient-driven dynamical system.

The construction admits a natural Bayesian interpretation.
Classical Moreau envelope smoothing corresponds to computing a Maximum A Posteriori (MAP) estimator under Gaussian perturbations.
In contrast, PGH follows the posterior mean, which can be expressed via Tweedie’s formula~\cite{kaipio2006statistical,kailath} as a gradient of the log-marginal density.
Replacing hard inner minimization with probabilistic averaging yields a tractable stochastic gradient estimator that aggregates information from a Gaussian neighborhood, biasing descent directions toward low-energy regions rather than relying solely on local curvature.

\paragraph{Contributions.}
Our contributions are summarized as follows. 
\begin{itemize}
    \item We reinterpret diffusion-style Gaussian smoothing as a homotopy mechanism for optimization, which establishes a bridge between classical Gaussian homotopy, Moreau envelope smoothing, and score-based diffusion. 
    \item We introduce \emph{Probabilistic Gaussian Homotopy (PGH)}, a probabilistic continuation framework that smooths the Boltzmann distribution associated with an objective rather than the objective itself.
    \item We establish a rigorous connection between this construction and Moreau envelope smoothing, showing that PGH corresponds to a posterior-mean (soft-min) variant of classical Moreau regularization.
    \item We derive an ordinary differential equation governing the evolution of minimizers along the homotopy and propose practical numerical schemes for tracking this path.
    % \item We provide theoretical insight through a \textcolor{red}{closed-form analysis in the quadratic case} and demonstrate strong empirical performance on standard nonconvex benchmarks.
\end{itemize}
% \vspace{0.5em}
In summary, we introduce a principled probabilistic continuation framework that connects Bayesian denoising, diffusion dynamics, and global nonconvex optimization.

\section{Background}

We briefly review three fundamental concepts which underlies our framework: Gaussian homotopy in optimization, Moreau envelope smoothing, and score-based diffusion.

\subsection{Gaussian Homotopy in Objective Space}

Homotopy continuation methods address nonconvex optimization by embedding a difficult objective into a family of progressively less smoothed surrogates.
Given $f : \mathbb{R}^n \to \mathbb{R}$, classical Gaussian homotopy constructs
\begin{equation}
F_\sigma(x) = (f * G_\sigma)(x)
= \int_{\mathbb{R}^n} f(y)\, G_\sigma(x-y)\, dy,
\end{equation}
where $G_\sigma$ is a Gaussian kernel with variance $\sigma^2$.
Equivalently,
\begin{equation}
F_\sigma(x)
= \mathbb{E}_{z \sim \mathcal{N}(0,I)}[f(x + \sigma z)].
\end{equation}
For large $\sigma$, $F_\sigma$ is smooth and often easier to optimize; as $\sigma \to 0$, $F_\sigma \to f$.

The continuation strategy tracks minimizers $x^\star(\sigma) \in \arg\min F_\sigma(x)$ while gradually decreasing $\sigma$.
Gradients are given by
\begin{equation}
\nabla F_\sigma(x)
= \mathbb{E}_{z}[\nabla f(x + \sigma z)],
\end{equation}
so descent directions correspond to a uniform average of gradients in a Gaussian neighborhood.

While effective in certain regimes, this formulation operates directly in objective space and averages function values and gradients \emph{uniformly}.
As a result, competing basins contribute equally to the descent direction, which may potentially blur structural differences between low- and high-energy regions.
This observation motivates exploring continuation directly in probability space rather than in objective space.

\subsection{Moreau Envelope and Soft-Min Smoothing}

Another classical smoothing construction is the Moreau envelope.
For $\lambda > 0$, the Moreau envelope of $f$ is defined as
\begin{equation}
M_\lambda f(x)
= \min_{y} \left\{
f(y) + \frac{1}{2\lambda} \|x - y\|^2
\right\}.
\end{equation}
The envelope replaces $f$ by a locally regularized minimization problem.
As $\lambda \to 0$, $M_\lambda f(x) \to f(x)$, while for larger $\lambda$ the envelope becomes smoother and may reduce nonconvexity.

A useful probabilistic perspective arises from the soft-min (log-sum-exp) operator:
\begin{equation}
\mathrm{softmin}_\lambda(g)
= -\lambda \log \int
\exp\!\left(-\frac{1}{\lambda} g(y)\right) dy.
\end{equation}
By Laplace’s principle, $\mathrm{softmin}_\lambda(g)$ approximates $\min_y g(y)$ as $\lambda \to 0$.
Applying this to
\begin{equation}
g(y) = f(y) + \frac{1}{2\lambda}\|x-y\|^2
\end{equation}
yields a smoothed approximation to the Moreau envelope that replaces hard minimization by probabilistic averaging.
This observation provides a bridge between convexification methods and probabilistic smoothing.

\subsection{Score-Based Diffusion and Gaussian Smoothing}

Score-based diffusion models construct a family of probability densities $\{p_t\}_{t\in[0,1]}$ by progressively adding Gaussian noise to data.
A common forward noising process takes the form
\begin{equation}
x_t = \alpha(t) x_1 + \beta(t) \varepsilon,
\quad \varepsilon \sim \mathcal{N}(0,I),
\end{equation}
where $\beta(t)$ increases with $t$.
Sampling is performed by integrating a reverse-time stochastic differential equation involving the score $\nabla_x \log p_t(x)$.

From a geometric viewpoint, diffusion defines a Gaussian homotopy in probability space:
the density evolves continuously from a simple Gaussian to a complex target distribution.
While diffusion models are typically interpreted as sampling algorithms, their structure naturally induces a continuation path across progressively de-smoothed distributions.

These three perspectives based on objective-space Gaussian homotopy, Moreau smoothing, and diffusion-based Gaussian perturbation, form the conceptual basis for the probabilistic homotopy framework developed in the next section.

\section{Probabilistic Gaussian Homotopy}

We now introduce \emph{Probabilistic Gaussian Homotopy (PGH)}, a continuation framework that operates directly in probability space rather than in objective space. 
The key idea is to deform the Boltzmann density associated with $f$ and track its minimizers as the smoothing vanishes.

\subsection{Probability-Space Homotopy Construction}

Let $f : \mathbb{R}^n \to \mathbb{R}$ be a possibly nonconvex objective and define its Boltzmann density
\begin{equation}
p_1(x) \propto \exp\!\left(-\frac{1}{\lambda(1)} f(x)\right).
\end{equation}

Instead of smoothing $f$ directly, we construct a time-dependent family of densities $\{p_t\}_{t \in [0,1]}$ by introducing Gaussian perturbations:
\begin{equation}
p_t(x)
\propto
\mathbb{E}_{z \sim \mathcal{N}(0,I)}
\left[
\exp\!\left(
-\frac{1}{\lambda(t)}
f\big(\alpha(t)x + \beta(t) z\big)
\right)
\right],
\label{eq:pt_definition}
\end{equation}
where $\alpha(t)$ and $\beta(t)$ define a smoothing schedule satisfying $\beta(1)=0$ and $\beta(0)>0$.

The corresponding energy functional is
\begin{equation}
F_t(x) = -\lambda(t)\log p_t(x).
\label{eq:Ft_definition}
\end{equation}

This construction defines a continuous deformation of the objective landscape.
At $t=1$, $\beta(1)=0$ and we recover the original objective,
\begin{equation}
F_1(x) = f(x).
\end{equation}
At $t=0$, strong smoothing renders $F_0(x)$ nearly constant, which yields a trivial optimization problem.
Thus, $\{F_t\}$ forms a homotopy between a flat landscape and the original nonconvex objective, as observed in Figure~\ref{fig:figpGH}.
Crucially, smoothing occurs in the Boltzmann density rather than directly in $f$, producing a continuation mechanism fundamentally different from objective-space Gaussian homotopy.

\subsection{Gradient Structure: Soft-Min Aggregation}

To understand the geometry of this homotopy, we examine the gradient of $F_t$.
Differentiating \eqref{eq:Ft_definition} yields
\begin{equation}
\nabla_x F_t(x)
=
\alpha(t)
\frac{
\mathbb{E}_z
\left[
\exp\!\left(
-\frac{1}{\lambda(t)} f(\alpha(t)x+\beta(t)z)
\right)
\nabla f(\alpha(t)x+\beta(t)z)
\right]
}{
\mathbb{E}_z
\left[
\exp\!\left(
-\frac{1}{\lambda(t)} f(\alpha(t)x+\beta(t)z)
\right)
\right]
}.
\label{eq:gradient_exact}
\end{equation}

The gradient is therefore a \emph{weighted average} of gradients in a Gaussian neighborhood of $\alpha(t)x$, where weights are proportional to $\exp(-f/\lambda)$.
This induces a soft-min (log-sum-exp) aggregation. In particular, \emph{perturbations with lower objective value contribute exponentially more to the descent direction.}

This contrasts with objective-space Gaussian homotopy, where gradients are averaged uniformly.
Here, the continuation dynamics are biased toward low-energy regions, naturally emphasizing promising basins while suppressing high-energy perturbations.

\subsection{Continuation Dynamics}

Let
\begin{equation}
x^\star(t) \in \arg\min_x F_t(x).
\end{equation}
Under suitable regularity conditions, minimizers evolve smoothly with $t$.
A natural mechanism for tracking this evolution is the non-autonomous gradient flow
\begin{equation}
\frac{d x(t)}{dt}
=
- \nabla_x F_t(x(t)).
\label{eq:gradient_flow}
\end{equation}
Here,~\eqref{eq:gradient_flow} defines a probability-space continuation dynamics.
As $t$ increases and smoothing decreases, trajectories follow the evolving minimizer path.

\subsection{Stochastic Approximation and Discrete Scheme}

In practice, expectations in \eqref{eq:gradient_exact} are approximated via Monte Carlo sampling.
Given $z_k \sim \mathcal{N}(0,I)$,
\begin{equation}
\nabla_x F_t(x)
\approx
\alpha(t)
\sum_{k=1}^K
w_k(x,t)
\nabla f(\alpha(t)x+\beta(t)z_k),
\end{equation}
where
\begin{equation}
w_k(x,t)
=
\frac{
\exp\!\left(-\frac{1}{\lambda(t)} f(\alpha(t)x+\beta(t)z_k)\right)
}{
\sum_j
\exp\!\left(-\frac{1}{\lambda(t)} f(\alpha(t)x+\beta(t)z_j)\right)
}.
\end{equation}
This estimator replaces hard inner minimization with soft probabilistic aggregation and parallelizes naturally across samples.

Discretizing \eqref{eq:gradient_flow} with time steps $t_k$ yields
\begin{equation}
x_{k+1}
=
x_k
-
\eta_k
\nabla_x F_{t_k}(x_k),
\end{equation}
which simultaneously performs gradient descent on $F_{t_k}$ and advances along the homotopy path.
We refer to this procedure as \emph{Probabilistic Gaussian Homotopy Optimization (PGHO)}.
Multiple trajectories may be evolved in parallel, enabling exploration from diverse initializations.
Algorithm~\ref{alg:alg1} summarizes the procedure.

\begin{algorithm}[t]
\caption{Probabilistic Gaussian Homotopy Optimization (PGHO)}
\begin{algorithmic}[1]
\State \textbf{Input:} Objective $f$, homotopy steps $T$, batch size $B$, samples $K$, step sizes $\{\eta_k\}$
\State Sample initial particles $x_i^0 \sim \mathcal{N}(0,I)$ for $i=1,\dots,B$
\For{$k=0$ to $T-1$}
    \State $t_k \leftarrow k/(T-1)$
    \For{$i=1$ to $B$}
        \State Sample $z_{i,1},\dots,z_{i,K} \sim \mathcal{N}(0,I)$
        \State Compute Monte Carlo objective
        \begin{equation}
        F_{t_k}^{(K)}(x_i^k)
        = -\lambda(t)
        \log\!\left(
        \frac{1}{K}
        \sum_{j=1}^K
        \exp\big(- \lambda(t_k)^{-1}f(\alpha(t_k)x_i^k+\beta(t_k)z_{i,j})\big)
        \right)
        \end{equation}
        \State Compute gradient
        \begin{equation}
        g_i^k = \nabla_x F_{t_k}^{(K)}(x_i^k)
        \end{equation}
        \State Update
        \begin{equation}
        x_i^{k+1} \leftarrow x_i^k - \eta_k g_i^k
        \end{equation}
    \EndFor
\EndFor
\State \textbf{Return} $\{x_i^T\}_{i=1}^B$
\end{algorithmic}
\label{alg:alg1}
\end{algorithm}

\paragraph{Computational Complexity.}
Each gradient evaluation requires $K$ evaluations of $f$ and $\nabla f$, yielding total cost ${\cal O}(T B K C_f)$, where $C_f$ denotes the cost of evaluating $f$ and its gradient. All operations parallelize over batch and noise samples.

\section{Connection to Moreau Envelopes and Bayesian Interpretation}

We now show that Probabilistic Gaussian Homotopy (PGH) can be interpreted as a probabilistic, finite-temperature generalization of classical Moreau smoothing, and admits a natural Bayesian interpretation.

\subsection{Soft Moreau Envelope}

We now formalize the relationship between Probabilistic Gaussian Homotopy (PGH) and a finite-temperature (log-sum-exp) relaxation of the Moreau envelope.
Recall that the Moreau envelope of $f$ with parameter $\lambda>0$ is defined as
\begin{equation}
M_\lambda f(x)
=
\min_{y}
\left\{
f(y) + \frac{1}{2\lambda}\|x-y\|^2
\right\}.
\label{eq:moreau}
\end{equation}
Replacing the hard minimum in \eqref{eq:moreau} by a soft-min (log-sum-exp) operator yields the finite-temperature surrogate
\begin{equation}
\widetilde{M}_\lambda f(x)
=
-\lambda \log
\int
\exp\!\left(
-\frac{1}{\lambda}
\left(
f(y) + \frac{1}{2\lambda}\|x-y\|^2
\right)
\right)
dy.
\label{eq:soft_moreau}
\end{equation}
We now show that PGH coincides with this soft Moreau envelope under a particular parameter choice.

\begin{theorem}[PGH as Soft Moreau Envelope]
Let $F_t(x)$ denote the PGH energy defined in~\eqref{eq:Ft_definition}. 
% \begin{equation}
% F_t(x)
% =
% -\lambda(t)
% \log
% \mathbb{E}_{z \sim \mathcal{N}(0,I)}
% \left[
% \exp\!\left(
% -\frac{1}{\lambda(t)} f(\alpha(t)x+\beta(t)z)
% \right)
% \right].
% \end{equation}
If $\alpha(t)=1$ and $\beta(t)=\sqrt{\lambda(t)}$, then
\begin{equation}
F_t(x)
=
\widetilde{M}_{\lambda(t)} f(x)
+ C(t),
\end{equation}
where $C(t)$ is a constant independent of $x$.
\end{theorem}

\begin{proof}
Starting from \eqref{eq:soft_moreau}, perform the change of variables
\begin{equation}
y = x + \sqrt{\lambda} z,
\quad
dy = \lambda^{n/2} dz.
\end{equation}
Then $\|x-y\|^2 = \lambda \|z\|^2$, and \eqref{eq:soft_moreau} becomes
\begin{equation}
\widetilde{M}_\lambda f(x)
=
-\lambda \log
\int
\exp\!\left(
-\frac{1}{\lambda} f(x+\sqrt{\lambda}z)
\right)
\exp\!\left(-\frac{1}{2}\|z\|^2\right)
dz
+ C,
\end{equation}
where $C$ absorbs normalization constants independent of $x$.
Recognizing the Gaussian density yields
\begin{equation}
\widetilde{M}_\lambda f(x)
=
-\lambda \log
\mathbb{E}_{z \sim \mathcal{N}(0,I)}
\left[
\exp\!\left(
-\frac{1}{\lambda} f(x+\sqrt{\lambda}z)
\right)
\right]
+ C.
\end{equation}
This matches the PGH energy with $\alpha=1$ and $\beta=\sqrt{\lambda}$.
\qed
\end{proof}

The theorem shows that PGH reduces to a finite-temperature (soft) Moreau envelope under the canonical scaling $\beta^2=\lambda$. 
Thus, PGH can be interpreted as a probabilistic relaxation of classical Moreau smoothing in which hard minimization is replaced by Gaussian-weighted averaging.

\subsection{Bayesian Interpretation and Posterior Mean Dynamics}

Similarly, we formalize the Bayesian interpretation of PGH and show that its gradient dynamics correspond to posterior-mean updates under Gaussian perturbations.

Consider the denoising model
\begin{equation}
x = y + \sqrt{\lambda}\, z,
\quad
z \sim \mathcal{N}(0,I),
\end{equation}
with Gibbs prior
\begin{equation}
\pi(y) \propto \exp(-f(y)).
\end{equation}
The posterior distribution of $y$ given $x$ is
\begin{equation}
p(y \mid x)
\propto
\exp\!\left(
-f(y) - \frac{1}{2\lambda}\|x-y\|^2
\right).
\label{eq:posterior}
\end{equation}
The Maximum A Posteriori (MAP) estimator of \eqref{eq:posterior} coincides with the Moreau minimizer \eqref{eq:moreau}. 
In contrast, PGH follows the posterior mean.

\begin{theorem}[Posterior Mean Identity]
Define the marginal density
\begin{equation}
p_\lambda(x)
=
\int
\exp(-f(y))
\,\mathcal{N}(x; y, \lambda I)\, dy,
\end{equation}
and the associated probabilistic envelope
\begin{equation}
F_\lambda(x)
=
- \log p_\lambda(x).
\end{equation}
Then the posterior mean satisfies
\begin{equation}
\mathbb{E}[y \mid x]
=
x - \lambda \nabla F_\lambda(x).
\label{eq:posterior_mean_identity}
\end{equation}
\end{theorem}

\begin{proof}
Differentiating $p_\lambda(x)$ under the integral sign yields
\begin{equation}
\nabla_x p_\lambda(x)
=
\int
\exp(-f(y))
\nabla_x \mathcal{N}(x; y, \lambda I)\, dy.
\end{equation}
Using the Gaussian score identity
\begin{equation}
\nabla_x \mathcal{N}(x; y, \lambda I)
=
-\frac{1}{\lambda}(x-y)\mathcal{N}(x; y, \lambda I),
\end{equation}
we obtain
\begin{equation}
\nabla_x \log p_\lambda(x)
=
-\frac{1}{\lambda}
\left(
x - \mathbb{E}[y \mid x]
\right).
\end{equation}
Rearranging yields \eqref{eq:posterior_mean_identity}.
\qed
\end{proof}

This theorem shows that gradient descent on the probabilistic envelope $F_\lambda$ corresponds exactly to a posterior-mean update:
\begin{equation}
x \leftarrow x - \eta \nabla F_\lambda(x)
\quad
\Longleftrightarrow
\quad
x \leftarrow (1-\eta)x + \eta\,\mathbb{E}[y \mid x].
\end{equation}
Thus, PGH replaces the MAP-based hard minimization of the Moreau envelope by posterior-mean dynamics driven by Gaussian perturbations. 
This finite-temperature update remains well-defined even in highly nonconvex regimes, where MAP estimators may be unstable or difficult to compute. PGH can therefore be interpreted as a posterior-mean generalization of Moreau smoothing.

\section{Related Work}
\paragraph{Gaussian Homotopy in Optimization.}
Gaussian homotopy (GH) methods have a long history in global optimization, often leveraging heat-equation-based smoothing to regularize nonconvex landscapes \cite{dunlavy2005homotopy,mobahi2015theoretical}. 
Classical GH constructs a family of smoothed objectives 
\(
F(x,t) = (f * G_{\sigma(t)})(x)
\)
by convolving the objective function $f$ with a Gaussian kernel \cite{mobahi2015theoretical}. 
Continuation is achieved by gradually reducing the smoothing parameter while tracking minimizers. 
More recent variants, such as Single-Loop GH (SLGH) \cite{iwakiri2022single}, improve computational efficiency by updating the smoothing parameter and optimization variables simultaneously. 
Despite these advances, existing GH methods operate strictly in \emph{objective space}, which average function values and gradients uniformly across perturbations \cite{mobahi2015theoretical,mobahi2015link}. 
In contrast, our approach performs continuation in \emph{probability space}, which smoothens the induced Boltzmann distribution 
\(
p(x) \propto \exp(-\frac{1}{\lambda}f(x))
\)
rather than the raw objective. 
This shift fundamentally alters the geometry of the descent direction, introducing a soft-min aggregation mechanism absent in classical GH.

\paragraph{Score-Based Diffusion Models.}
Score-based diffusion models have recently achieved remarkable success in high-dimensional generative modeling by learning to reverse a forward Gaussian noising process \cite{vincent2011connection,dhariwal2021diffusion}. 
These methods construct a continuum of intermediate densities and sample by integrating a reverse-time stochastic differential equation driven by the score function $\nabla_x \log p_t(x)$. 
While diffusion models are primarily studied in the context of sampling and image generation, their structure naturally defines a Gaussian homotopy in probability space. 
Connections between diffusion and inverse problems have begun to emerge, but the interpretation of diffusion dynamics as a principled continuation method for \emph{global optimization} remains largely unexplored. 
Our work makes this connection explicit by showing that probability-space smoothing induces a minimizer path analogous to classical homotopy methods.

\paragraph{Moreau Envelopes and Bayesian Estimation.}
The Moreau envelope is a central tool in convex analysis and proximal algorithms~\cite{heaton2024global,osher2023hamilton,nocedal1999numerical,boydBook,di2026operator,di2025monte}, providing a smooth regularization of nonsmooth or nonconvex objectives. 
It has recently been applied in large-scale and federated optimization settings \cite{t2020personalized,heaton2024global,osher2023hamilton,meng2025recent}. 
From a statistical viewpoint, the envelope corresponds to a Maximum A Posteriori (MAP) estimator under Gaussian perturbations \cite{tensiam}. 
Recent work has further developed log-integral (log-sum-exp) approximations to infimal convolutions via Laplace’s method~\cite{tibshirani2025laplace,darbon2021bayesian,darbon2021connecting}, which replaces hard minimization with finite-temperature exponential averaging and yields a posterior-mean interpretation. 
PGH aligns with this probabilistic perspective. In particular, it can be viewed as a finite-temperature, Boltzmann-weighted generalization of Moreau smoothing. 
However, rather than approximating proximal maps, we use this log-integral structure to define a continuation path in probability space and derive dynamics that track minimizers as smoothing vanishes. 
This connects Laplace-based smoothing, Bayesian denoising, and diffusion-based homotopy within a unified framework.

\paragraph{Stochastic Gradient Estimation and Graduated Optimization.}
Our stochastic gradient estimator builds on classical results for gradient estimation of expectations \cite{kaipio2006statistical,dalalyan2017theoretical}. 
Unlike standard stochastic gradient descent, which evaluates $\nabla f$ at a single point, PGH computes a weighted average of gradients over a Gaussian neighborhood. 
This mechanism shares similarities with entropy-regularized methods~\cite{chaudhari2019entropy}, but differs in that the weights are derived from the Boltzmann energy, which emphasizes low-energy perturbations exponentially. 
As a result, the continuation dynamics selectively amplify promising descent directions rather than uniformly averaging local curvature.

\paragraph{Global optimization algorithms.}
Random-search methods, such as Pure Random Search (PRS) \cite{Brooks1958RandomMethods}, sample points from the domain and keep the best value observed.
Population-based heuristics such as Differential Evolution (DE) \cite{StornPrice1997DE} and Particle Swarm Optimization \cite{KennedyEberhart1995PSO} maintain a set of candidate solutions
and update them using randomized recombination or velocity-style rules, respectively.
Basin Hopping (BH) \cite{WalesDoye1997BasinHopping,OlsonHashmiMolloyShehu2012BasinHopping} combines random perturbations with local refinement steps, accepting or rejecting moves based on the resulting objective value.
Annealing-based methods \cite{KirkpatrickGelattVecchi1983SimulatedAnnealing} allow occasional uphill moves with a temperature-controlled probability, enabling escape from local minima.
Implicit Filtering (IF) \cite{Kelley2011ImplicitFiltering} is a derivative-free local method based on stencil sampling with a sequence of decreasing step sizes, and is often used for noisy or non-smooth objectives.
Finally, the Covariance Matrix Adaptation Evolution Strategy (CMA-ES) \cite{HansenOstermeier2001CMAES} adapts a Gaussian search distribution by learning its covariance over time,
yielding efficient search directions on ill-conditioned landscapes.
In contrast to these black-box global optimization heuristics, our method exploits gradient information and performs a continuation procedure through a
probabilistic smoothing of the induced Boltzmann distribution, producing a Boltzmann-weighted aggregation of perturbed gradients.
This yields a structured descent direction that remains stable in highly multimodal landscapes.

\begin{figure}
    \centering
    \begin{tabular}{ccc}
    \includegraphics[width=0.5\linewidth]{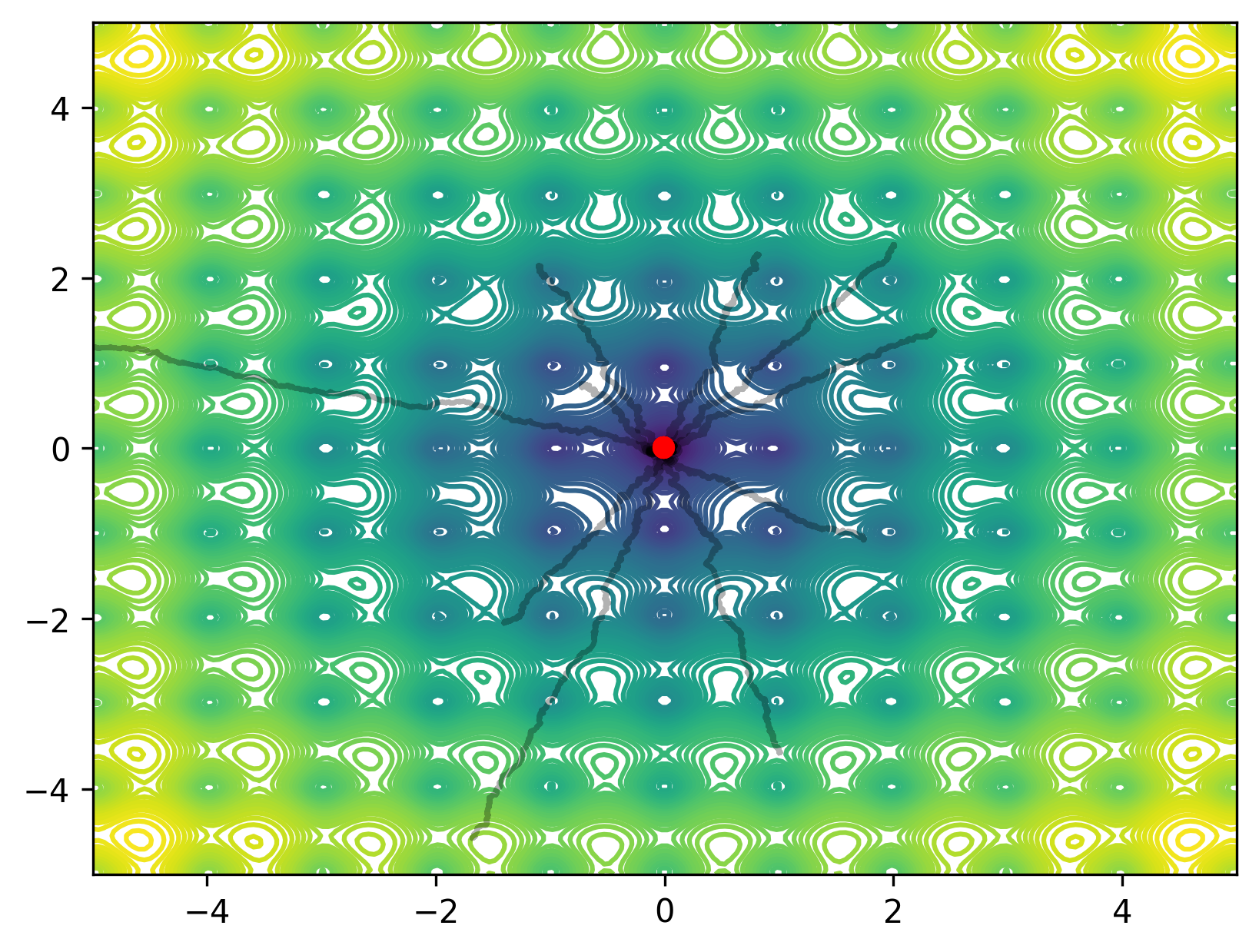} &
    \includegraphics[width=0.5\linewidth]{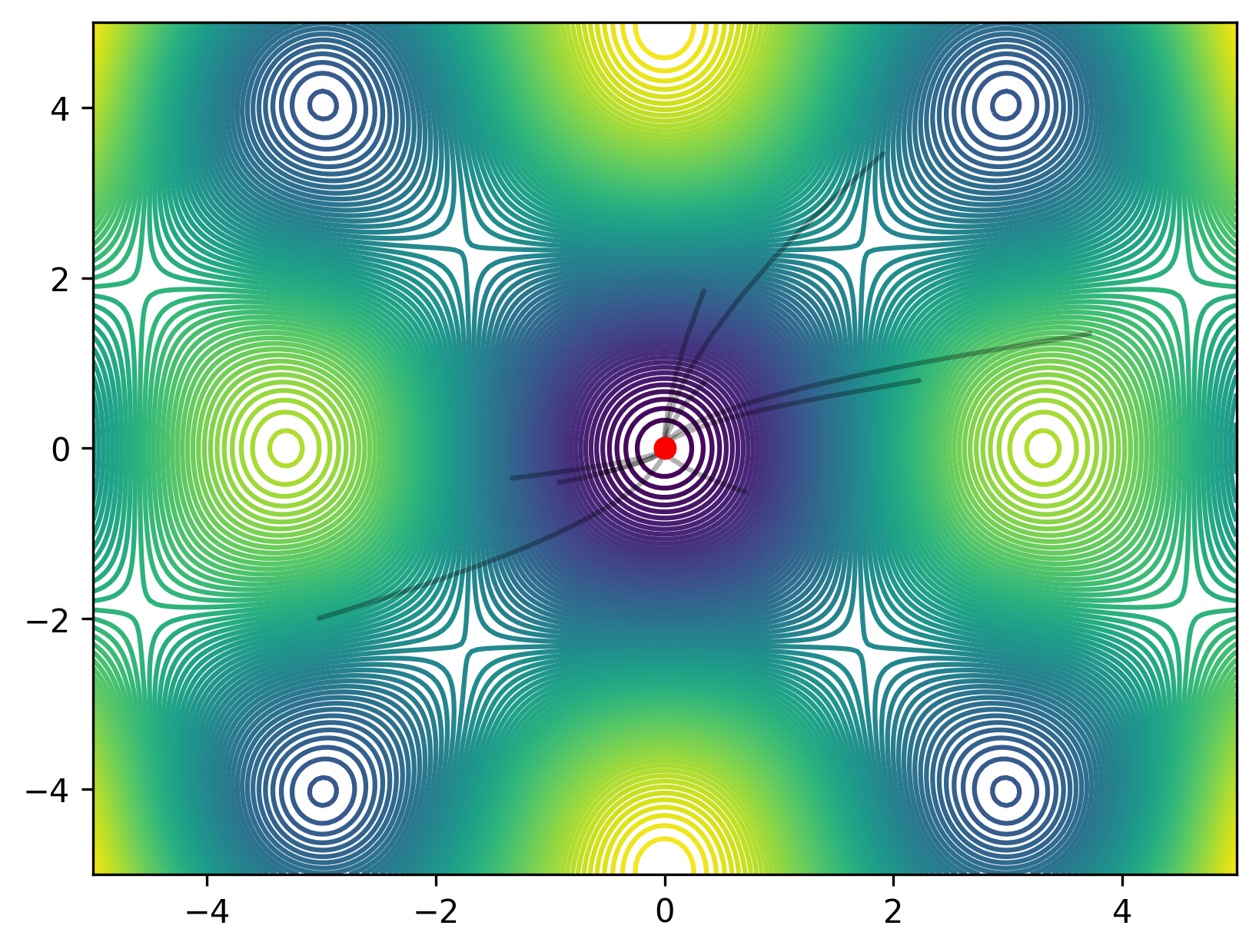} \\
    Ackley &  Griewank
    \end{tabular}
    \caption{Ackley and Griewank plotted for the 2D case with the trajectories obtained by the PGH algorithms starting from different points }
    \label{fig:model_problems}
\end{figure}

\section{Numerical Experiments}

\subsection{Benchmark Problems}
To evaluate the capability of our method to converge on challenging non-convex objectives, we consider standard benchmark functions:
Ackley, Griewank, Alpine1, and Levy, taken from the Virtual Library of Simulation Experiments \cite{SurjanovicBingham2013VLSF}.
Figure~\ref{fig:model_problems} visualizes representative of those functions in the $2$D case, and shows representative convergence trajectories from several random initializations.

We next compare our method with a series of global optimization algorithms, including classical Gaussian Homotopy (GH),
Covariance Matrix Adaptation Evolution Strategy (CMA-ES)~\cite{HansenOstermeier2001CMAES},
Differential Evolution (Diff.Eval.)~\cite{StornPrice1997DE}, Basin Hopping (BasinHop.)~\cite{WalesDoye1997BasinHopping,OlsonHashmiMolloyShehu2012BasinHopping},
( Simulated / Dual ) Annealing (Anneal.)~\cite{KirkpatrickGelattVecchi1983SimulatedAnnealing},
Implicit Filtering (Impl.Filt.) \cite{Kelley2011ImplicitFiltering}, Particle Swarm (Part.Swarm)~\cite{KennedyEberhart1995PSO}, and Pure Random Search (PRS)~\cite{Brooks1958RandomMethods}.
We also report two variants of our approach: \emph{PGH-GD} and \emph{PGH-Adam}, which share the same probabilistic homotopy
objective but differ only in the base update rule used for the inner step (gradient descent vs.\ Adam \cite{KingmaBa2015Adam}).
Table~\ref{tab:tableX_avg_evals_success} summarizes the results. Overall, our method is the fastest across these benchmarks, consistently
reaching the target tolerance with substantially fewer evaluations than the competing baselines.

We further report convergence behavior and scalability on these benchmarks.
% Figure~\ref{fig:convergence}
% shows the optimization traces over multiple runs in dimension $n=10$, illustrating that our method
% typically reaches the success threshold in a small number of evaluations and with low variability across runs. 
Figure~\ref{fig:success_vs_dim_ackley}
reports success rate as a function of dimension, where our method remains highly consistent even as the dimension increases, while some competing
methods degrade substantially in higher dimensions. This highlights that the probabilistic aggregation in our method
provides a stable search direction in settings where other global optimizers may struggle to scale.
Additional experimental details are provided in the Appendix.

% \begin{figure}
%     \centering
%     \begin{tabular}{ccc}
%     \includegraphics[width=0.5\linewidth]{Figures/ackley.png} &
%     % \includegraphics[width=0.33\linewidth]{Figures/rastrigin.png} &
%     \includegraphics[width=0.5\linewidth]{Figures/griewank.png} \\
%     Ackley &  Griewank
%     \end{tabular}
%     \caption{Ackley and Griewank plotted for the 2D case with the trajectories obtained by the PGH algorithms starting from different points }
%     \label{fig:model_problems}
% \end{figure}

\begin{table}[t]
\centering
\caption{Mean number of function evaluations required to reach success (``N'' indicates no successful runs), where success is defined as $f(x) < 5\times10^{-2}$. Each run uses a maximum budget of $2 \times 10^5$ function evaluations. Results are averaged over 30 runs at dimension $n=10$.}
\scriptsize
\setlength{\tabcolsep}{2.5pt}
\renewcommand{\arraystretch}{1.05}
\resizebox{\linewidth}{!}{%
\begin{tabular}{lcccccccccc}
\hline
Function & PGH-GD & PGH-Adam & GH & CMAES & Diff.Eval. & BasinHop. & Anneal. & Impl.Filt. & Part.Swarm &  PRS \\
\hline
Ackley    & 205 & 601 & 3{,}894 & 804  & 6{,}683  & N & 97{,}646 & 240   & 3{,}663 &  N \\
% Rastrigin & 159{,}020        & 142{,}953      & N  & N    & 56{,}786 & N & N       & N   & 101{,}040     &  N \\
Griewank  & 183 & 631 & 3{,}248 & 1{,}088 & 8{,}937 & 14{,}120 & 99{,}785 & 315 & 4{,}891 &  N \\
Alpine1   & 192 & 557 & 3{,}479 &  1{,}176 & 11{,}367 & N & 97{,}195 & 308   & 4{,}187 &  N \\
Levy      & 3{,}067 & 562 & 3{,}141 & 817  & 5{,}731  & N & 85{,}933 & N & 3{,}146 &  N \\
\hline
\end{tabular}%
}
\label{tab:tableX_avg_evals_success}
\end{table}

% \begin{figure}[h]
%     \centering
%     \begin{tabular}{cc}
%         \includegraphics[width=0.45\textwidth]{Figures/convergence_plots/Ackley_PGH_convergence.png} &
%         \includegraphics[width=0.45\textwidth]{Figures/convergence_plots/Alpine1_PGH_convergence.png} \\
%         \includegraphics[width=0.45\textwidth]{Figures/convergence_plots/Griewank_PGH_convergence.png} &
%         \includegraphics[width=0.45\textwidth]{Figures/convergence_plots/Levy_PGH_convergence.png} \\
%     \end{tabular}
%     \caption{Convergence of PGH-Adam on four benchmark functions (dim=10). 
%     Solid line shows median best objective value over 20 runs; 
%     shaded region shows interquartile range. 
%     Dashed line marks the success threshold ($5 \times 10^{-2}$).}
%     \label{fig:convergence}
% \end{figure}

\begin{figure}[h]
    \centering
    \includegraphics[width=0.6\textwidth]{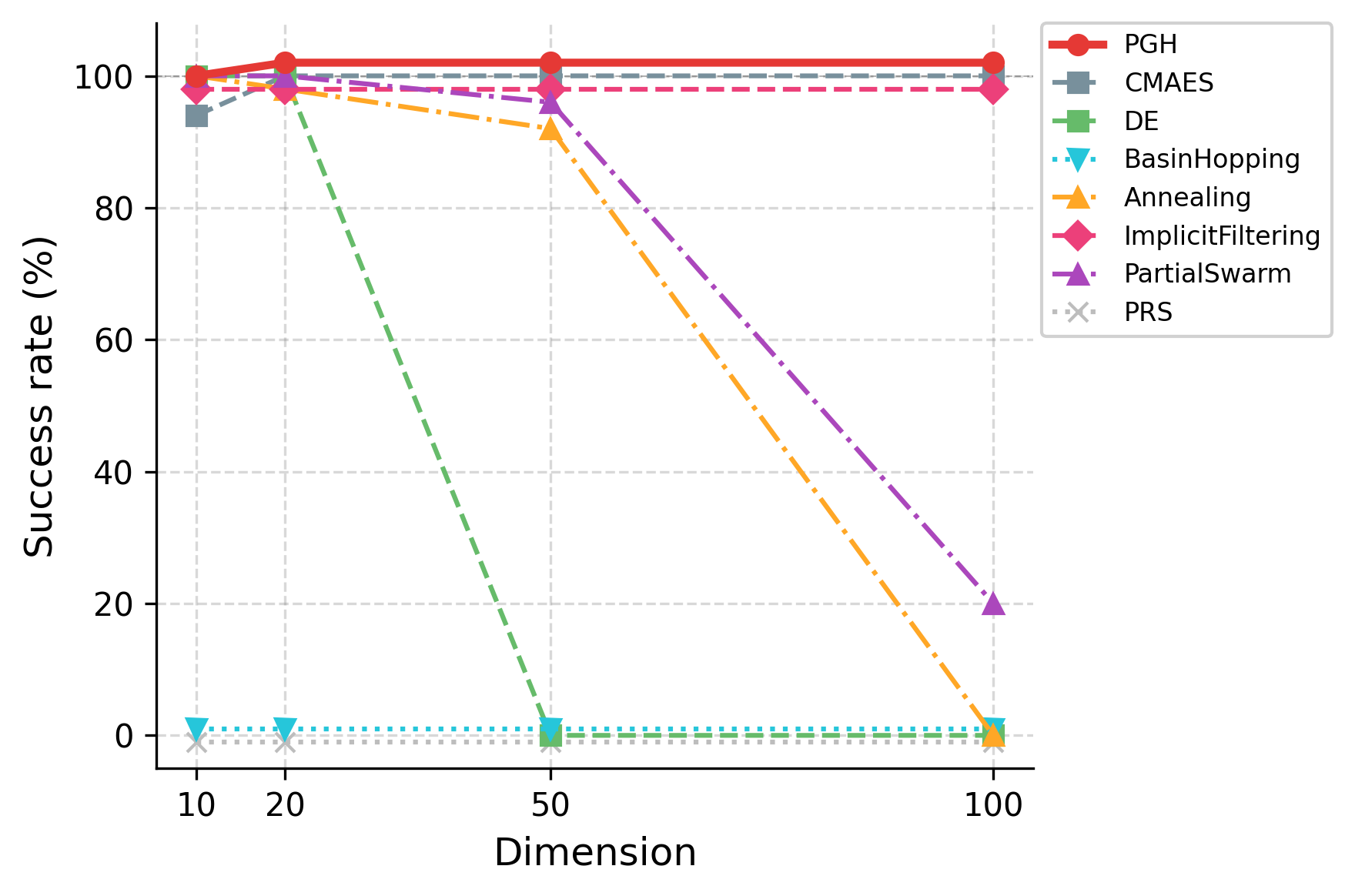}
    \caption{Success rate vs.\ dimension on Ackley. PGH maintains high success rate
    across all dimensions, while competing methods degrade significantly on higher dimensions. Budget fixed at $10^5$ function evaluations, averaged over 30 runs.}
    \label{fig:success_vs_dim_ackley}
\end{figure}

\subsection{Sparse Recovery}
\label{sec:sparse_recovery}

We next consider sparse recovery from undersampled linear measurements, which provides a standard and challenging nonconvex inverse problem. We observe
\begin{equation}
    y = A x^\star + \eta,
\end{equation}
where $A \in \mathbb{R}^{m \times n}$ is a sensing matrix with $m < n$, $x^\star \in \mathbb{R}^n$ is a sparse ground-truth signal, and $\eta$ is additive noise. Our goal is to recover $x^\star$ from $y$ by solving a sparsity-promoting regression problem.

In particular, we study the $\ell_0$-regularized least-squares formulation
\begin{equation}
    \min_{x \in \mathbb{R}^n}
    \; \frac{1}{2}\|Ax-y\|_2^2 + \lambda \|x\|_0,
\end{equation}
which balances data fidelity with sparsity of the recovered signal \cite{natarajan1995sparse}. While there are many possible relaxations for this problem \cite{tibshirani1996lasso,mohimani2009sl0}, in this work we focus on a smooth approximation to the $\ell_0$ regularization. 
\begin{equation}
    f_\tau(x)
    =
    \frac{1}{2}\|Ax-y\|_2^2
    +
    \lambda \sum_{i=1}^n \left(1 - e^{-x_i^2/\tau^2}\right),
\end{equation}
where $\tau > 0$ controls the sharpness of the approximation. As $\tau \to 0$, the penalty approaches the counting function $\|x\|_0$, while for fixed $\tau$ it remains differentiable and can be handled by gradient-based methods.

We compare two variants of our method, \emph{PGH-GD} and \emph{PGH-Adam}, against the corresponding baseline optimizers \emph{GD} and \emph{Adam}. We evaluate the methods along a regularization path over $\lambda$ and focus on the main qualitative outcomes.
{\em Note that we do not compare different $\ell_0$ relaxations but rather the use of different optimization algorithms to solve this particular relaxation.}
Additional implementation details are provided in the Appendix.

Figure~\ref{fig:sparse_recovery} summarizes the sparse recovery results. Figure~\ref{fig:sparse_recovery}(a) shows the tradeoff between the data misfit $\frac{1}{2}\|Ax-y\|_2^2$ and the smooth $\ell_0$ term $\sum_i (1-e^{-x_i^2/\tau^2})$. This plot plays the role of an L-curve: better solutions lie closer to the lower-left region, where one achieves both low reconstruction error and stronger sparsity. The PGH-based methods trace a more favorable frontier than the baseline methods over most of the path, indicating an improved sparsity--fidelity tradeoff.

Figure~\ref{fig:sparse_recovery}(b) shows the attained objective value $f_\tau(x)$ as a function of $\lambda$. Both PGH variants obtain lower objective values than standard GD and Adam. The improvement is particularly visible in the intermediate range of $\lambda$, where the balance between fitting the measurements and enforcing sparsity is most delicate. This suggests that the probabilistic homotopy mechanism helps the optimizer navigate the nonconvex landscape more effectively and avoid poorer local minima.

Overall, these experiments indicate that our methods provide a consistent improvement over the corresponding first-order baselines for sparse recovery. Empirically, PGH leads to solutions that are both more competitive in objective value and better balanced in terms of sparsity and data fidelity.

\begin{figure}[t]
    \centering
    \begin{tabular}{cc}
        \includegraphics[width=0.48\linewidth]{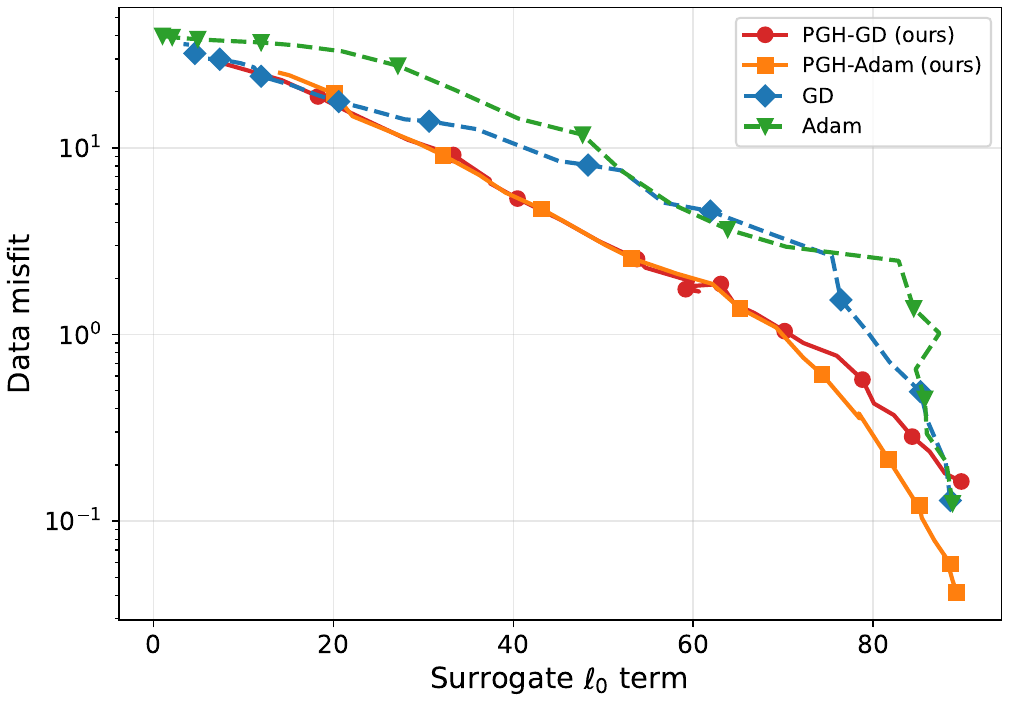} &
        \includegraphics[width=0.48\linewidth]{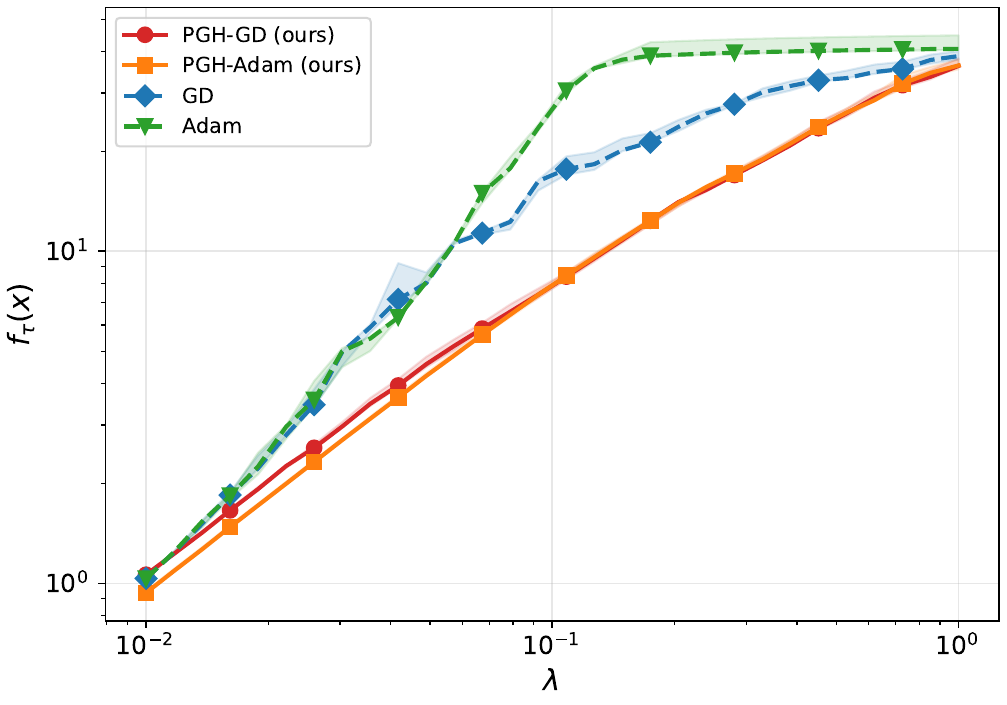} \\
        (a) Data misfit vs.\ surrogate $\ell_0$ term
        &
        (b) Final objective $f_{\tau}(x)$ vs.\ $\lambda$
    \end{tabular}
    \caption{Sparse recovery results along the regularization path. Left: tradeoff between reconstruction fidelity and sparsity. Right: attained smooth objective value as a function of the regularization parameter. In both plots, the PGH variants improve over the corresponding GD and Adam baselines.}
    \label{fig:sparse_recovery}
\end{figure}

\section{Conclusion}

In this work, we introduced Probabilistic Gaussian Homotopy (PGH), a continuation framework that operates in probability space rather than directly in objective space, leading to Boltzmann-weighted aggregation of perturbed gradients that emphasizes low-energy regions and provides a principled search direction for difficult nonconvex problems. We further established connections to finite-temperature Moreau-type smoothing and a posterior-mean interpretation, and demonstrated across a range of nonconvex optimization settings that PGH yields strong empirical performance, including favorable results on benchmark test problems and improved behavior on sparse recovery. Overall, these results suggest that probability-space continuation is a promising and effective alternative to classical smoothing approaches for challenging nonconvex optimization.

\clearpage
\newpage


\appendix
\section{Experimental Details}

\subsection{Benchmark Problems}
\label{app:benchmark_details}

We evaluate the method on four standard nonconvex benchmark functions: Ackley, Griewank, Alpine1, and Levy. Unless otherwise stated, experiments are run in dimension $n=10$, with a maximum budget of $2\times 10^5$ function evaluations and $30$ independent random trials. Success is defined by the threshold $f(x) < 5\times 10^{-2}$. All methods are initialized from a random point sampled uniformly from the standard box constraints of the corresponding benchmark function, as stated in \cite{SurjanovicBingham2013VLSF}. Specifically, we use the domains $[-5,5]^n$ for Ackley, $[-600,600]^n$ for Griewank, and $[-10,10]^n$ for Alpine1 and Levy.
All methods are evaluated under the same function-evaluation budget and the same success criterion as in the main text. Results are reported as the mean number of function evaluations required to reach the target threshold, with ``N'' indicating that no successful run was observed within the allotted budget.

\subsubsection{Benchmark Objective Functions.}

We evaluate the method on four standard nonconvex benchmark functions.

\paragraph{Ackley function.}
\begin{equation}
\begin{aligned}
f_{\mathrm{Ackley}}(x)
&=
-20 \exp\!\left(-0.2 \sqrt{\frac{1}{n}\sum_{i=1}^n x_i^2}\right) \\
&\quad
- \exp\!\left(\frac{1}{n}\sum_{i=1}^n \cos(2\pi x_i)\right)
+ 20 + e .
\end{aligned}
\end{equation}

\paragraph{Griewank function.}
\begin{equation}
f_{\mathrm{Griewank}}(x)
=
\frac{1}{40}\sum_{i=1}^n x_i^2
- \prod_{i=1}^n \cos\!\left(\frac{x_i}{\sqrt{i}}\right)
+ 1 .
\end{equation}

\paragraph{Alpine1 function.}
\begin{equation}
f_{\mathrm{Alpine1}}(x)
=
\sum_{i=1}^n \left|x_i \sin(x_i) + 0.1 x_i\right| .
\end{equation}

\paragraph{Levy function.}
\begin{equation}
\begin{aligned}
f_{\mathrm{Levy}}(x)
&=
\sin^2(\pi w_1)
+ \sum_{i=1}^{n-1} (w_i-1)^2 \left(1 + 10\sin^2(\pi w_i + 1)\right) \\
&\quad
+ (w_n-1)^2 \left(1 + \sin^2(2\pi w_n)\right) .
\end{aligned}
\end{equation}

where
\begin{equation}
w_i = 1 + \frac{x_i - 1}{4}, \qquad i=1,\dots,n.
\end{equation}

The global minimizers are $x^\star = 0$ for Ackley, Griewank, and Alpine1, and $x^\star = (1,\dots,1)$ for Levy.

\subsubsection{PGH Variants and Hyperparameters.}

We report two variants of our method, \emph{PGH-GD} and \emph{PGH-Adam}, which share the same probabilistic homotopy objective and differ only in the base optimizer used in the update step. In both cases, the PGH gradient is estimated by Monte Carlo sampling with antithetic perturbation pairs. All PGH runs use $K=4$ Monte Carlo samples per step; the only exception is PGH-GD on Levy, where we use $K=16$ to reduce variance. The homotopy parameter is increased from $0$ to $1$ over a prescribed number of homotopy steps, after which the method continues at full homotopy until the evaluation budget is exhausted. Across the benchmark problems, the number of homotopy steps was tuned per problem and typically ranged from $50$ to $200$. Both PGH-GD and PGH-Adam use cosine learning-rate schedules over the homotopy phase, so the tuned values are interpreted as initial learning rates. These initial learning rates were also selected separately for each problem and optimizer, and ranged approximately from $10^{-1}$ to $5\times 10^{1}$, reflecting differences in scaling and landscape geometry across the benchmark functions. In all cases, the iterates are constrained to remain within the benchmark search box.

\subsection{Sparse Recovery}
\label{app:sparse_recovery_details}

For the sparse recovery experiments, we consider signals of ambient dimension $n=10{,}000$ with sparsity level $k=100$, observed through undersampled linear measurements with sampling ratio $m/n=0.15$ and additive Gaussian noise of standard deviation $\sigma=0.01$. The smooth $\ell_0$ parameter is set to $\tau=0.05$, and results are averaged over $3$ independent trials. We sweep the regularization parameter over a log-spaced grid with $\lambda \in [10^{-2},1]$ using $30$ values.

For the PGH methods, we use $K=4$ Monte Carlo samples per iteration, with a maximum of $10{,}000$ iterations, corresponding to a total budget of $4\times 10^4$ function evaluations. PGH uses learning rate $0.05$. We use perturbation scale $\varepsilon=0.1$ and minimum homotopy parameter $t_{\min}=0.37$. The homotopy parameter is updated using a square-root schedule, so that it increases gradually from $t_{\min}$ to $1$ over the course of the run. We choose $t_{\min}>0$ to avoid the near-zero homotopy regime, where gradients become very small and optimization can progress slowly. In the PGH implementation, antithetic perturbation pairs are used in the Monte Carlo gradient estimate to reduce variance. For the baseline methods, GD uses learning rate $0.05$ and Adam uses learning rate $0.01$ for stability. All methods use cosine annealing schedules, and are initialized from the zero vector.

\clearpage
\newpage

\begin{credits}
% \subsubsection{\ackname} A bold run-in heading in small font size at the end of the paper is
% used for general acknowledgments, for example: This study was funded

% \subsubsection{\discintname}
% It is now necessary to declare any competing interests or to specifically
% state that the authors have no competing interests. Please place the
% statement with a bold run-in heading in small font size beneath the
% (optional) acknowledgments,
% for example: The authors have no competing interests to declare that are
% relevant to the content of this article. Or: Author A has received research
% grants from Company W. Author B has received a speaker honorarium from
% Company X and owns stock in Company Y. Author C is a member of committee Z.
\subsubsection{\discintname}
The authors have no competing interests to declare that are
relevant to the content of this article.
\end{credits}

\bibliographystyle{splncs04}
\bibliography{biblio}
%% Note that this preceding line implies that you store your BibTeX references in a file called 'mybibliography.bib'. If you instead store your references in a file with a different name, for instance 'references.bib', the preceding line should read '\bibliography{references}'. Whatever you do, DO NOT put the file name extension .bib inside the \bibliography command; this will trip up LaTeX compilers. 
%
% If you do not want to use BibTeX, you can also type up the bibliography exactly as you see fit, using the following structure:
% % Note that this number 8 reserves an amount of space (equal to the natural width of the given number) for the label of your references; if you have more than 9 references, you will want to change this number to 18. If you have more than 19 references, this number is best changed to 88. If you have more than 99 references, I salute you.




% LNCS Homepage, \url{http://www.springer.com/lncs}, last accessed 2023/10/25
% \end{thebibliography}

% \clearpage
% \newpage
% \appendix
% \input{supplement}

\end{document}